\documentclass{article}

\usepackage{amsmath}
\usepackage{amssymb}
\usepackage[]{algorithm2e}
\usepackage{amsfonts}
\usepackage{cite}
\usepackage{amsmath}
\usepackage{amssymb}
\usepackage{amsthm}
\usepackage{algorithm2e}
\usepackage{graphicx}
\usepackage{multirow}
\usepackage{array}
\usepackage[toc,page]{appendix}
\newtheorem{theorem}{Theorem}
\newtheorem{lemma}{Lemma}
\graphicspath{ {figures/} }
\title{A Review of Nonnegative Matrix Factorization Methods for Clustering}
\author{Ali Caner T\"urkmen \\
		  \small{Department of Computer Engineering}\\
		  \small{Bogazici University, Istanbul, Turkey}\\
		  \small{\texttt{caner.turkmen@boun.edu.tr}}}

\DeclareMathOperator*{\argmin}{arg\,min}

\begin{document}

\maketitle

\begin{abstract}
	Nonnegative Matrix Factorization (NMF) was first introduced as a low-rank matrix approximation technique, and has enjoyed a wide area of applications. Although NMF does not seem related to the clustering problem at first, it was shown that they are closely linked. In this report, we provide a gentle introduction to clustering and NMF before reviewing the theoretical relationship between them. We then explore several NMF variants, namely Sparse NMF, Projective NMF, Nonnegative Spectral Clustering and Cluster-NMF, along with their clustering interpretations.
\end{abstract}

\section{Introduction}

Clustering, the problem of partitioning observations with high intra-group similarity, has always been one of the central themes in unsupervised learning. Although a wide variety of algorithms for clustering applications have been studied in literature, the subject still remains an active avenue for research. 

In fact, it is possible to formulate clustering as a matrix decomposition problem. This formulation leads to interesting interpretations as well as novel algorithms that benefit from favorable computational properties of numerical linear algebra.

Popularized by Lee and Seung \cite{lee_learning_1999}, Non-negative Matrix Factorization (NMF) has turned into one of the primary tools for decomposing data sets into low-rank factorizing matrices in order to yield a parts-based representation. It has been shown not long after, that NMF and variants are well-performing alternatives to well-known clustering algorithms and that close theoretical links exist between the two problems. Drawing on this link, researchers have offered a variety of perspectives and different methods in applying NMF for clustering. 

This report aims to make a gentle introduction to how the clustering problem can be interpreted in a matrix factorization setting. A variety of NMF formulations will be presented, along with the intuition about their clustering interpretations. The algorithmic solutions and implementation details of these techniques will be left out with pointers to the relevant literature. 

This report is structured as follows. We introduce the clustering problem and several solution strategies in Section \ref{sec:clus}. We then introduce NMF and establish the link between $k$-means and NMF, in Section \ref{sec:nmf}. We then work on this link to introduce several NMF formulations geared towards solving the problem in Section \ref{sec:variants}. Namely, we cover Sparse NMF, Projective NMF, Nonnegative Spectral Clustering and Cluster-NMF. We conclude the report in Section \ref{sec:conclusion}.

\subsection{A Word on Notation}

Capital letters of the Latin and Greek alphabets denote matrices (e.g. $A, \Gamma, \Sigma$), while lowercase bold typeface letters denote vectors (e.g. $\mathbf{x}$). Vectors that correspond to a matrix are column vectors of that matrix, indexed accordingly. That is, $\mathbf{x_1}$ is the first column of matrix $X$. $X_{ij}$ will be used to denote the $i$-th row, $j$-th column element of matrix $X$. Script capital letters denote sets (e.g. $\mathcal{C}$).

Throughout the document, $||.||$ will be used for the $L_2$ norm of a vector, $||.||_F$ for the Frobenius norm of a matrix. $tr(.)$ stands for the \emph{trace} operator, i.e. the sum of the elements along the diagonal of a square matrix. $diag(.)$ is used in a similar sense to popular linear algebra software, and will denote a diagonal matrix constructed on the vector or list of scalars provided.

Finally, vector dimensions will also be reused. In the data clustering problem, $n$ will be introduced as the number of observations, $m$ as the number of features (i.e. the dimension of each observation), and $k$ as the number of clusters or corresponding \emph{low-rank} interpretation in NMF.


\section{Clustering}
\label{sec:clus}

In this section, we provide a gentle introduction to the clustering problem, and introduce the notation we will use when discussing NMF applications. 

Given a set of data points $\mathbf{x}_i \in \mathcal{X}$ for $1 \le i \le n$, clustering algorithms aim to find partitionings of the set such that the \emph{similarity} within each partition is maximized and similarity between different partitions is minimized. This general problem definition can be attacked with various techniques to generate a variety of different partitionings on the same set.

This simple idea has a wide array of application areas including management science, business intelligence, pattern recognition, web search, biostatistics. For example, applied to a database of customers clustering algorithms yield groups of similar customers which may better respond to targeted marketing campaigns. In musical information retrieval, clustering can be used to group together songs for similar audiences.

Clustering, or \emph{cluster analysis}, is an \emph{unsupervised} learning problem. That is, clustering techniques address unlabeled data, and try to work with different objective functions that do not involve a \emph{ground truth}.

Although it is difficult to provide a definitive taxonomy of clustering algorithms, it is useful to provide some definitions on how they are often categorized. The two primary families of clustering algorithms are partitioning methods and hierarchical methods. Partitioning methods partition the data into $k$ groups, and iteratively relocate the observations until some dissimilarity between groups and similarity within groups is achieved. This is in contrast to hierarchical clustering, where the algorithm decomposes the data set hierarchically into partitions. In the agglomerative approach, single data points are merged into groups one by one until a termination condition holds. Starting from observations and building partitions by merging small groups, these methods are also called \emph{bottom-up}. A contrast is divisive or \emph{top-down} methods which start with all nodes in the same cluster and partition clusters until some termination condition occurs.

Another divide between clustering methods we will find useful is \emph{hard} clustering versus \emph{soft} clustering. Hard clustering methods adopt \emph{exclusive cluster assignment}, in that at any step of the algorithm one observation can belong to one and at most one cluster. Soft clustering is a relaxation on this constraint, and each observation can belong to different clusters in proportion to a set of weights. 

In the following sections, we will introduce the most popular partitioning problem, $k$-means and Kernel $k$-means, a variant. This will be useful as we will reference these techniques numerous times when discussing NMF applications in clustering. 

\subsection{$k$-means}
\label{sec:clus_kmeans}

$k$-means is perhaps the most popular clustering method, and the first one that comes to mind. Suppose a data set $\mathcal{X}$ contains $n$ observations denoted by $\mathbf{x}_i \in \mathbb{R}^m$. $k$-means aims to find a partitioning of this data set $\mathcal{C}_1, \mathcal{C}_2, ..., \mathcal{C}_k$ such that $\mathcal{C}_i \subset \mathcal{X}$ and $\mathcal{C}_i \cap \mathcal{C}_j = \emptyset$ for all $1 \le i,j \le k$. The technique uses an objective function given below to minimize dissimilarity within $k$ disjoint subsets $\mathcal{C}_i$:

\begin{equation}
\label{eq:kmeans_objective}
\min_{{\mathcal{C}_1, \mathcal{C}_2,..., \mathcal{C}_k}} \sum_{i=1}^{k} \sum_{\mathbf{x} \in \mathcal{C}_i} ||\mathbf{x} - \mu_i||^2
\end{equation} 

where $\mu_i$ is the \emph{centroid} of data points assigned to cluster $\mathcal{C}_i$. In $k$-means, as the name implies, the centroid of cluster $\mathcal{C}_i$ is defined as the mean of all points assigned to it. In this regard, $k$-means tries to minimize \emph{within cluster variation}, or simply the sum of squared error within each cluster.

Algorithms for solving $k$-means are some of the earliest machine learning algorithms that have survived to date. A high-level description of the naive algorithm can be found in Algorithm \ref{algo:kmeans}. Early derivations can be found in \cite{macqueen_methods_1967,lloyd_least_1982}.

\begin{algorithm}[H]
	\SetAlgoLined
	\KwData{$k$: number of clusters \\ $\mathcal{X}$: set of data points}
	\KwResult{Set of $k$ clusters $\{\mathcal{C}_1, \mathcal{C}_2,..., \mathcal{C}_k\}$}
	arbitrarily choose $k$ objects as initial cluster centroids\;
	\While{cluster centroids change}{
		assign each point $\mathbf{x} \in \mathcal{X}$ to cluster $\mathcal{C}_i$ with the closest (by Euclidean distance) cluster centroid $\mu_i$ \;
		recalculate $\mu_i$ as the mean of data points assigned to $\mathcal{C}_i$, for all $1 \le i \le k$ \;
	}
	\caption{$k$-means algorithm}
	\label{algo:kmeans}
\end{algorithm}

Expressing $k$-means in vectorized form will help greatly in \emph{kernelizing} a solution as well as expressing it in matrix factorization form. 

Let us define \emph{cluster membership matrix} $B \in \{0, 1\}^{n \times k}$, such that the matrix element $B_{ij} = 1$ if observation $i$ is assigned to cluster $j$, and 0 otherwise. We will refer to this matrix numerous times throughout this document, so it makes sense to explore its properties. First, observe that each row of $B$ can have only one element taking the value 1, and all other elements must be equal to 0. This is due to the definition of hard clustering. As such, the columns of matrix $B$ are orthogonal.

Consequently we also have $\sum_i B_{ij} = |\mathcal{C}_j|$. We now introduce the matrix $$D = diag(1/|\mathcal{C}_1|, 1/|\mathcal{C}_2|,..., 1/|\mathcal{C}_k|)$$ Observe that we could have equivalently written $|\mathcal{C}_i| = ||\mathbf{b}_i||^2$ for all $i$. 

Let us calculate $D^\frac{1}{2}$, defined as $D^\frac{1}{2}_{ij} = \sqrt{D_{ij}}$. Note that calculating $BD^\frac{1}{2}$, we have effectively normalized the columns of $B$ as well, so they are orthonormal. Equivalently, $(BD^\frac{1}{2})^TBD^\frac{1}{2} = D^\frac{1}{2}B^TBD^\frac{1}{2} = I$.

Finally, we arrange our observations as the columns of a data matrix $X$ such that 

$$
X = [\mathbf{x}_1, \mathbf{x}_2, ..., \mathbf{x}_n]
$$

Briefly revisit (\ref{eq:kmeans_objective}). Note that, due to the definition of $k$-means, we define the cluster centroids $\mu_i$ as

$$\mu_i = \frac{1}{|\mathcal{C}_i|} \sum_{\mathbf{x} \in \mathcal{C}_i} \mathbf{x}$$

We are now equipped with all the tools we need to express $k$-means in vectorized form. Observe that the matrix product $XBD$ yields a matrix which has cluster centroid vectors $\mu_i$ arranged as its columns. We are looking to express the Euclidean distance of each observation from the cluster center it is assigned to. We can simply do this by calculating $XBDB^T$. We have now effectively cast the cluster centroids to $\mathbb{R}^{m \times n}$, and arranged a matrix where each column $i$ corresponds to the cluster centroid that the data item $\mathbf{x}_i$ is assigned to. We can now rewrite the clustering objective (\ref{eq:kmeans_objective}) in the minimization problem below:

\begin{equation}
	\label{eq:kmeans_vectorized}
	\min_B ||X - XBDB^T||_F^2
\end{equation}

This is the same sum of squared error given in the original function. Also observe that the objective function is only dependent on $B$, since $D$ is only defined as the inverse squared column norms of $B$ arranged on the diagonal of a matrix.

Advancing this notation further, we know that we could write $||A||_F^2 = tr(A^TA)$ for any matrix $A$. Using this property, we write the same problem in the form:

\begin{equation*}
	\label{eq:kmeans_vec_trace}
	\min_B tr((X - XBDB^T)^T(X - XBDB^T))
\end{equation*}

As shown in \cite{zha_spectral_2001,ding_equivalence_2005}, we can show that solving this form is equivalent to solving

\begin{equation*}
	\label{eq:kmeans_vec_trace_2}
	\min_B tr(X^TX) - tr(X^TXBDB^T)
\end{equation*}

This derivation is  more involved, and is presented in Appendix \ref{app:a}, as Lemma \ref{lemma:tracemax}. Observe that the first term in the new objective function is a constant, and the problem can safely be reexpressed as:

\begin{equation}
	\label{eq:kmeans_vec_tracemax}
	\max_B tr(X^TXBDB^T)
\end{equation}

We will see in later sections that this set of derivations lend themselves to extensions for matrix factorization forms as well as spectral relaxation and application of the kernel trick. We can now move on to introduce, by means of this notation, how the $k$-means problem can be extended with kernel functions.

\subsection{Kernel $k$-means}
\label{sec:clus_kernel_kmeans}

At its simplest form, $k$-means clustering is only capable of calculating spherical clusters. An extension to $k$-means for addressing more complex, non-linearly bounded clusters is by means of the \emph{kernel trick}, introduced in this section. After a brief introduction to the kernel trick, we will move to define kernel $k$-means. 

In many machine learning tasks, a given set of data points $\mathbf{x} \in \mathcal{X}$ are transformed into feature vectors via a \emph{feature map}, $\phi: \mathcal{X} \rightarrow \mathcal{F}$ where $\mathcal{F}$ is the feature space \cite{smola_tutorial_2004}. However, for describing highly non-linear interactions and patterns in data, the feature map may end up producing feature vectors of much higher dimension. To avoid the associated computational cost, one may construct a function $K: \mathcal{X} \times \mathcal{X} \rightarrow \mathbb{R} $ such that 

\begin{equation*}
K(\mathbf{x}, \mathbf{y}) = \phi(\mathbf{x})^T\phi(\mathbf{y})
\end{equation*}

Having constructed this function, often via the use of domain knowledge, it will now take only $O(n^2)$ time in terms of inner products to compute a Gram matrix, where $n$ is the number of observations. In computing this matrix, the data items have been \emph{implicitly} transformed to a feature space of possibly infinite dimension. 

More intuitively, kernel functions often have implications as measures of similarity between vectors $\mathbf{x}$ and $\mathbf{y}$. Machine learning methods which rely on the kernel trick to perform highly non-linear separations are called \emph{kernel methods}. 

When clusters are spherical, dense and linearly separable, the naive $k$-means algorithm\footnote{We refer here to the algorithm given in Algorithm \ref{algo:kmeans}. Note that $k$-means is a \emph{problem}, not an algorithm. An appropriate name for it could be the Lloyd algorithm \cite{lloyd_least_1982}} can be expected to give a fair representation of clusters. However, if the clusters are of arbitrary shape, and are only non-linearly separable, then a different approach is required. It will be shown that the $k$-means algorithm can be \emph{kernelized} in order to achieve that effect.

\begin{figure}[h]
	\begin{center}
		\includegraphics[width=11cm]{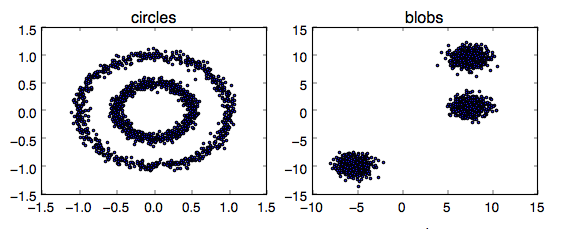}
	\end{center}
	\caption{Linear Separability of Data}
	\label{fig:circles_blobs}
\end{figure}

For motivation, take Figure \ref{fig:circles_blobs}. It is clear that while the second \emph{blobs} data set is linearly separable into meaningful clusters, a non-linear decision boundary is required to partition the first data set into the two \emph{circles}. 

It is for this purpose we extend the standard $k$-means formulation to kernelized form, as given in \cite{dhillon_kernel_2004, dhillon_unified_2004}. Let us introduce the matrix $$\Phi = [\phi(\mathbf{x}_1), \phi(\mathbf{x}_2), ..., \phi(\mathbf{x}_n)]$$

Intuitively, $\Phi$ is a matrix that has the observations $\mathbf{x} \in \mathcal{X}$, transformed via the feature map $\phi$ to arbitrary dimension in its columns. 

Recall the line of derivations that led to (\ref{eq:kmeans_vec_tracemax}). Without applying the kernel trick, we would have used the matrix $\Phi$ instead of $X$ in each step of the process, having explicitly transformed its features. We would have then ended up with the trace maximization problem:

\begin{equation}
	\label{eq:kmeans_kernel_tracemax}
	\max_B tr(\Phi^T\Phi BDB^T)
\end{equation}

However, note that by definition, $\Phi^T\Phi$ is the \emph{kernel matrix}, such that each element $(\Phi^T\Phi)_{ij} = \phi(\mathbf{x}_i)^T\phi(\mathbf{x}_j) = K(\mathbf{x}_i, \mathbf{x}_j)$. Given this fact, we never need to explicitly compute the matrix $\Phi$. We could just compute the kernel matrix $\Phi^T\Phi$ with the kernel function, and safely plug it in place of the linear inner product matrix $X^TX$.

The kernel trick has a profound implication. Without the need for ever computing features in the feature space $\mathcal{F}$, the kernel matrix has the power of \emph{implicitly} mapping data to arbitrary dimension. However, note that calculating the kernel matrix and performing matrix operations on it requires additional computational cost, as the kernel matrix has $n^2$ elements whereas the data matrix had $mn$ elements. Especially for cases where the number of observations $n$ is very large, kernel methods pose a disadvantage as they become increasingly difficult to compute.

\section{Non-Negative Matrix Factorization}
\label{sec:nmf}

In machine learning, approximating a matrix by two factorizing low-rank matrices has many demonstrated benefits. Among these benefits is discovering structure in data, as well as reducing dimensionality and making way for better generalization.

Perhaps one of the most popular methods geared towards low-rank approximation is Principal Component Analysis (PCA), a close relative of Singular Value Decomposition (SVD). At a high level, PCA works to identify the next best vector (or \emph{component}) through the data that accounts for the highest variance \cite{Shlens2014}. Once this is done, it then repeats the process for the residual variance. As a direct and analytical result of applied linear algebra, PCA helps identify structure and reduce dimensionality.

Popularized in \cite{lee_learning_1999}, Non-negative Matrix Factorization (NMF) is a low-rank approximation technique which introduces the constraint that the data matrix and the factorizing matrices are nonnegative. By allowing only additive linear combinations of components with nonnegative coefficients (i.e. a conical combination), NMF inherently leads to a parts-based representation. Compared to PCA, which is a holistic representation model, NMF leads to a much more intuitive and interpretable representation.

Formally, NMF is characterized by the following factorization:

\begin{equation}
\label{eq:nmf}
X \approx WH 
\end{equation}

where $X, W, H \ge 0$, $X \in \mathbb{R}^{m \times n}, W \in \mathbb{R}^{m \times k}, H \in \mathbb{R}^{k \times n}$, and $k$ is the number of components (the \emph{rank}) in which the data matrix will be represented. Naturally, it makes sense that $ k < \min(m,n) $.

The approximation problem presented in (\ref{eq:nmf}) is often formulated as an optimization problem of the form $$\min_{W, H \ge 0} D(X || WH)$$where $D(A||B)$ is a \emph{divergence} function. The most popular choice for the divergence is the Euclidean distance between $X$ and $WH$, in which case the optimization problem becomes:

\begin{equation}
\label{eq:nmfopt}
\min_{W,H \ge 0} ||X - WH||_F^2 = \sum_{i,j} (X - WH)_{ij}^2 
\end{equation}

Several strategies and algorithms have been presented in literature towards the solution of this problem, including Constrained Alternating Least Squares in the original Paatero and Tapper presentation \cite{paatero_positive_1994}\footnote{in which the problem was called ``Positive Matrix Factorization''}, multiplicative update rules presented in \cite{lee_algorithms_2001}, projected gradient methods \cite{hoyer_non-negative_2004, lin_projected_2007}.

We dedicate the next few paragraphs to building intuition about the implications of NMF. Assume each column of $X$ is an image of $m$ pixels. Then, using the approximation $X \approx WH$, one may say that the columns of $W$ correspond to a \emph{basis image}. Each column of $H$ on the other hand, corresponds to an encoding of each image in terms of the basis images. 

Nonnegativity constraints on $X$, $W$, and $H$ implicitly make for a more interpretable and parts-based representation of the data. $W, H \ge 0$ implies that each element $X_{ij} = \mathbf{w}_i\mathbf{h}^T_j$, is a weighted addition of positive vectors. By limiting both the basis and the encoding to nonnegative values, NMF forces the factorization to an additive weighted structure (encoding), of nonnegative building blocks (bases). We will see in Section \ref{sec:nmf_app_cluster} that this is a useful constraint in clustering.

In contrast, PCA finds a basis composed of the eigenvectors of the covariance matrix $\frac{1}{n}XX^T$ corresponding to the $k$ greatest singular values, guaranteeing the best representation in lower-rank \cite{Shlens2014}. However, the eigenvectors of the covariance matrix are not necessarily nonnegative. Intuitively, PCA bases do not represent an additive linear combination of the features, and the weight in one principal component can be canceled out by another. This prevents PCA bases to be readily interpreted as parts-of-whole representations.

In turn, the nonnegativity in NMF implies that its use is constrained to cases where the data matrix is composed of nonnegative elements. In many real world scenarios, however, the nonnegativity of data points is inherent, such as pixel intensities in image processing, signal intensities, chemical concentrations, etc. 

\section{NMF Applications to Clustering}
\label{sec:nmf_app_cluster}

\subsection{Intuition}
\label{sec:cluster_intuition}

Having mentioned that NMF can be used to ``discover structure in data'', we already have the intuition that it must have interpretations in a clustering setting.

We had introduced the $k$-means problem earlier in Section \ref{sec:clus_kmeans}, and expressed it in vectorized form. We will now elaborate that this form is easily interpretable in a matrix factorization framework. First, recall (\ref{eq:kmeans_vectorized}):

$$
\min_{B, (BD^\frac{1}{2})^TBD^\frac{1}{2} = I} ||X - XBDB^T||_F^2
$$

Note that we have made the constraint that $B$ has orthogonal columns explicit. Let's elaborate on this optimization problem. Evidently, we are trying to approximate a data matrix, $X \approx XBDB^T$. If the rank of B is unconstrained, a trivial solution would be to take $B=I$, in which case $D=I$ would hold by definition. This would be equivalent to assigning each observation to its own cluster. However, clearly we would like to summarize the data in $k < min(m,n)$ clusters.

Compare the approximation problem $X \approx XBDB^T$ to (\ref{eq:nmf}). Assume our data matrix is $X \ge 0$. Then, by definition, the matrix of cluster centroids $XBD \ge 0$. Then we could compare $XBD$ to the \emph{basis} or $W$ in the NMF formulation, and $B^T$ to $H$, the \emph{encodings}. If such a representation is achievable, we have a solid way of learning clusters via favorable computational properties of numerical linear algebra.

However, a critical constraint in the clustering problem prevents us from jumping to this conclusion. We defined each row of $B$ as 0 except one element, which is 1. From here we found that $D^\frac{1}{2}B^TBD^\frac{1}{2}=I$. NMF, on the other hand, does not constrain $H^T$ to have orthogonal columns which would give us the needed clustering interpretability.

Let us dive deeper into this problem. What does $H^T$ having orthogonal columns imply? We have argued that row-wise sparsity induces orthogonality of columns. This leads to the intuition that relaxing this orthogonality leads to denser rows. Remember each row of $B$ is interpreted as the cluster assignment vector for some observation. A dense vector implies that the observation is not characterized as belonging to one cluster, but expressed as a weighting of several cluster centroids. If the original goal is clustering interpretability however, this is not a desirable outcome. 

Then, can we solve NMF while constraining the right factorizing matrix to have the same structure as $B^T$? It turns out this is also an intractable problem, and the orthogonality constraint must be relaxed in order to make the problem solvable in polynomial time. 

We are then left with the following intuition. If we can find ways to constrain factorizing matrices to near-orthogonality, NMF provides the ground for clustering interpretations. Then simply selecting the maximum element's index from the rows of $B$, we can directly use NMF methods for producing cluster assignments.

\subsection{Kernel $k$-Means and Symmetric NMF Equivalence}
\label{sec:cluster_equiv}
Ding et al. in \cite{ding_equivalence_2005} have shown a stronger relationship between symmetric NMF and spectral $k$-means clustering. They further demonstrate that solving the optimization problem presented in (\ref{eq:kmeans_kernel_tracemax}) with additional constraints one can achieve symmetric NMF formulation, while the orthogonality of the factorizing matrices are preserved. In this section, we follow their footsteps in demonstrating the same equivalence. 

Before moving forward, note that \emph{symmetric} NMF is simply the factorization a symmetric nonnegative matrix $A$ to factors $A \approx GG^T$.

To pave the way, we work with our usual notation introduced in (\ref{eq:kmeans_vectorized}): $$ \min_B ||X - XBDB^T||^2_F $$

We had further demonstrated, with the help of Lemma \ref{lemma:tracemax}, that this was equivalent to solving

\begin{equation}
	\label{eq:tracemax}
	 \max_{BD^\frac{1}{2}} tr((BD^\frac{1}{2})^TX^TXBD^\frac{1}{2}) \,\, s.t. \, (BD^\frac{1}{2})^TBD^\frac{1}{2} = I
\end{equation}

In fact, this completes the expression of the $k$-means problem introduced in Section \ref{sec:cluster_intuition}, in \emph{spectral relaxation} notation similar to \cite{zha_spectral_2001,ding_equivalence_2005}. We had shown in Section \ref{sec:clus_kernel_kmeans} that the inner product kernel matrix can safely be replaced with other kernel matrices $\Phi^T\Phi$. 

From this point on, we introduce a slight change in notation. We will call this kernel matrix $A = \Phi^T\Phi$. We will absorb $D^{\frac{1}{2}}$ into $B$, and introduce $G = BD^{\frac{1}{2}}$. Note that the definition of $B$ and $D$ immediately lead to, $G^TG = I$. 

We can now demonstrate that solving symmetric NMF on the kernel matrix, such that $A \approx GG^T$ is equivalent to Kernel $k$-means. 

\begin{theorem}
	\label{theorem:symmetric_kernel}
	(Ding, He and Simon) NMF of $A = GG^T$ is equivalent to Kernel $k$-means with strict orthogonality constraint relaxed.
\end{theorem}

\begin{proof} Observe that equivalent to maximizing (\ref{eq:tracemax}), one could minimize:
	
	\begin{subequations}
		\begin{align*}
		G^* &= \argmin_{G^TG = I, G \ge 0} -2 tr(G^TAG) \\
		&= \argmin_{G^TG = I, G \ge 0} -2 tr(G^TAG) + ||G^TG||_F^2 \\
		&= \argmin_{G^TG = I, G \ge 0} ||A||_F^2 -2 tr(G^TAG) + ||G^TG||_F^2 \\
		&= \argmin_{G^TG = I, G \ge 0} ||A - GG^T||_F^2 
		\end{align*}
	\end{subequations}
	
Relaxing $G^TG = I$ completes the proof.
\end{proof}

As such, $k$-means with orthogonality relaxed (i.e. \emph{soft} $k$-means) and symmetric NMF of the pairwise inner product matrix $X^TX$ are theoretically equivalent. By use of kernel matrices, this is easily generalized to kernel $k$-means.

However, relaxing the orthogonality constraint $G^TG = I$ may result in the density issue discussed in Section \ref{sec:cluster_intuition}. To that end, we must establish that near-orthogonality or row-wise sparsity of $G$ is enforced. We present another theorem in shown in \cite{ding_equivalence_2005} to demonstrate symmetric NMF preserves the necessary near-orthogonality.

\begin{theorem}
	(Ding, He and Simon) Optimizing $\min_G ||A - GG^T||_F^2$ retains the near-orthogonality $G^TG = I$
\end{theorem}

\begin{proof}
	
	First observe that the solution $G^* = \argmin_G ||A - GG^T||_F^2$ is not unique. That is, multiple alternatives of $G^*$ are available for the case $A \approx GG^T$. The next few equations will imply that among competing alternatives for $G$, solving for the above problem will favor near-orthogonal $G$. 
	
	We had already introduced the cost function:
	
	$$
	J = ||A||_F^2 -2 tr(G^TAG) + ||G^TG||_F^2
	$$
	
	Assuming $A \approx GG^T$ we can write:
	
	\begin{subequations}
		\begin{align*}
			J &\approx ||A||_F^2 - 2 tr (GG^TA) + ||G^TG||_F^2 \\
			J &\approx ||A||_F^2 - 2 tr (A^TA) + ||G^TG||_F^2 \\
			J &\approx -||A||_F^2 + ||G^TG||_F^2 \\
		\end{align*}
	\end{subequations}

	Then, the resulting problem can alternatively be written as:
	
	$$
	min_{G \ge 0, A \approx GG^T} ||G^TG||_F^2
	$$
	
	Note that:
	
	\begin{equation}
	\label{eq:gtg_decompose}
	||G^TG||_F^2 = \sum_{ij} (G^TG)^2_{ij} = \sum_{i \ne j}( \mathbf{g}_i^T \mathbf{g}_j)^2 + \sum_{i} ||\mathbf{g_i}||^4
	\end{equation}
	
	Minimizing the first term in the above addition leads to the desired near-orthogonality of columns, i.e. $\mathbf{g}_i^T \mathbf{g}_j \approx 0$ for all $1 \le i < j \le k$. The second term however, cannot be 0 as $A \approx GG^T$. Note that:
	
	\begin{equation}
	\sum_{ls} A_{ls} \approx \sum_{ls} (GG^T)_{ls} = \sum_{lis} G_{li} G_{is} = \sum_{i} |\mathbf{g}_{i}|^2
	\end{equation}		
	
	where $|.|$ denotes the $L_1$ norm. This implies $|\mathbf{g}_{i}| > 0$ and consequently $||\mathbf{g}_{i}|| > 0$. We are then left to conclude that minimizing $J$ preserves near-orthogonality in the columns of $G$.
\end{proof}

As discussed in Section \ref{sec:cluster_intuition}, this near-orthogonality (i.e. row-wise sparsity) is essential in a clustering interpretation. Shown above, expressing a somewhat relaxed $k$-means as symmetric NMF preserves this requirement, thus the two techniques are theoretically equivalent. Hence we have established that $k$-means and NMF have a strong theoretical link. We can now move to introduce various NMF methods that have been shown to perform well in clustering applications.

\section{NMF Methods for Clustering}
\label{sec:variants}

In the last section, we demonstrated a high-level link between $k$-means and nonnegative matrix factorization. We then extended this argument to demonstrate that relaxing the \emph{hard clustering} constraint on $k$-means, we are able to show a theoretical equivalence between the method and symmetric NMF on the kernel matrix. 

Now, drawing on this theoretical link, we will introduce several NMF formulations and corresponding algorithms; and how their results have a clustering interpretation.

\subsection{Sparse NMF}

As discussed above, the matrix $H$ can be interpreted as posterior cluster membership probabilities, and sparsity on its individual rows must be imposed to achieve interpretability in clustering with NMF. 

The natural idea that follows this line of argument is adding sparsity constraints or regularization to the cost-function optimized during NMF. One such formulation of the cost function is introduced in \cite{kim_sparse_2008}:

\begin{equation}
\label{eq:sparse_nmf_cost}
	\min_{W,H} \frac{1}{2} \left[ ||A - WH^T ||_F^2 + \eta ||W||_F^2 + \beta \sum_{j=1}^{n} |H(j,:)|^2  \right] s.t. W,H > 0
\end{equation}

where $|H(j,:)|$ refers to the $L_1$ norm of the $j$-th row of $H$. This regularization could be applied in many forms, including putting $||H||_F^2$, or $||\mathbf{h}_i||$ instead of the $|H(j,:)|^2$ term. However, we seek row-wise sparsity in the right-factorizing matrix, and this formulation will fulfill that goal.

Before moving to introduce Kim and Park's solution strategy for the SNMF problem introduced above, let us introduce the standard NMF solution via Alternating Non-Negativity Constrained Least Squares (ANLS). Given the objective function $$ \min_{W,H} ||X - WH||_F^2 $$ iteratively solving the following two nonnegative least squares problems converges to a stationary point of the objective function:

\begin{subequations}
	\label{eq:anls_noreg}
	\begin{align}
		&\min_{W \ge 0} ||H^TW^T - X^T||_F^2 \\
		&\min_{H \ge 0} ||WH - X||_F^2
	\end{align}
\end{subequations}

Note that the original cost function of NMF is non-convex, and non-convex optimization algorithms in this regard only guarantee stationarity of limit points, as is the case with NMF/ANLS. In other words, the algorithm is prone to converging to local minima, rather than the global minimum.

The authors of \cite{kim_sparse_2008} solve the above subproblems via the active set method introduced in detail in \cite{kim_nonnegative_2008}. In this paper, we will not focus on the internals of the minimization algorithms with respect to nonnegativity constraints, but rather the solution strategies with added regularization for sparsity. 

Take (\ref{eq:sparse_nmf_cost}). It's significance with respect to the clustering application is discussed in \cite{kim_sparse_2008}, along with superior results presented. The cost function is minimized under the following two ANLS update rules:

\begin{subequations}
	\begin{align}
		&\min_{H \ge 0} \left|\left| \binom{W}{\sqrt{\beta}\mathbf{e}_{k}}H - \binom{A}{0_n} \right|\right|^2_F  \\
		&\min_{W \ge 0} \left|\left| \binom{H^T}{\sqrt{\eta}I_{k}}W^T - \binom{A^T}{0_{k \times m}} \right|\right|^2_F 
	\end{align}
\end{subequations}

where $\mathbf{e}_{k} \in \mathbb{R}^{1 \times k}$ is a vector of all ones, $0_{k} \in \mathbb{R}^{1 \times k}$ is a vector of all zeros, $0_{k \times m} \in \mathbb{R}^{k \times m}$ is a matrix of all zeros and $I_k$ is the $k \times k$ identity matrix.

These two subproblems are directly linked to the original subproblems in (\ref{eq:anls_noreg}). We are only embedding the regularization terms in the factorizing matrices. The row $\mathbf{e}_k$ ensures that additional cost is incurred on the $L_1$ norm of rows of $H$. Similarly, $I_k$ stacked under $H$ simply ensures $||I_kW^T||_F^2 = ||W||_F^2$ is minimized.

Compared to \emph{spectral} methods that will be introduced in later sections, the main advantage of SNMF is that there is no need to calculate the affinity matrix $X^TX$, which may be computationally costly. Furthermore, the regularization terms introduce extra control parameters $\eta, \beta$ that implementers can use to trade-off accuracy versus interpretability. Other applications of SNMF have been discussed in \cite{kim_sparse_2006,kim_sparse_2007}.

\subsection{Projective NMF}
\label{sec:pnmf}

Projective NMF (PNMF) was introduced by Yuan and Oja in \cite{yuan_projective_2005}, for deriving basis vectors more suited towards learning localized representations. In a later publication \cite{yuan_projective_2009}, Yuan et al. demonstrate that PNMF is more closely related to $k$-means clustering than ordinary NMF equivalence discussed in Section \ref{sec:cluster_intuition}. In this section, we introduce the PNMF formulation, discuss its equivalence to $k$-Means clustering, and provide update rules for solving the approximation problem under Euclidean distance. 

Given a data matrix $X$, we are looking for a projection such that the columns of $X$ are projected to some subspace of $\mathbb{R}^m$, and are still accurate representations of the original matrix. In other words, we are looking for a matrix $P \in \mathbb{R}^{m \times m}$ such that $X \approx PX$. 

Now note that for any symmetric projection matrix of rank $k$ there exists $P = GG^T$, such that $G^TG = I$ and $G \in \mathbb{R}^{m \times k}$ (see Lemma \ref{lemma:eigendecomp}). We could then equivalently write $$X \approx GG^TX$$Let us go through the implications of this approximation one by one. $G^TX$ effectively defines the data items $\mathbf{x}_i$ in the dimension $\mathbb{R}^r$. Writing $GG^TX$, we simply cast this back to $R^{m \times n}$. In solving for $X \approx GG^TX$ we're simply looking to reduce the data matrix rank to $k$, and still find a good representation of the original matrix. Another interpretation is that we find prototypical features from $X$ that best represent the data in lower rank.

We already know one good $G$ that is capable of doing this. In fact, we can ensure that we can get the best representation possible in lower rank by setting $G = \hat{U}$, where $\hat{U} \in \mathbb{R}^{m \times k}$ is \emph{left singular vectors} of $X$ with the $k$ greatest singular values. This is a natural result of one of the staples of applied linear algebra: Singular Value Decomposition. 

However, singular vectors of $X$ are not necessarily nonnegative. $\hat{U}^TX$ does not represent features built from additive combinations of the features of the original data matrix $X$. This lacks the parts based representation leading to interpretability we required in Section \ref{sec:nmf}. 

Alternatively, we could cast the problem $X \approx GG^TX$ in an NMF framework to find a basis that is both sparse and parts based. We can now introduce Projective NMF (PNMF). Rewritten as an optimization problem, minimizing Euclidean distance as a divergence, PNMF algorithms solve:

\begin{equation}
	min_{G \ge 0} ||X - GG^TX||_F^2
\end{equation}

Let us connect this argument to clustering. Above, we argued that PNMF finds prototypical features (or rather groups of features) that best represent data in lower rank. In clustering, our goal is similar, we aim to find prototypical cluster centroids, or \emph{groups of observations} that represent the original data the best. One simple trick is now enough to connect PNMF to clustering: perform PNMF on the transposed data matrix $X^T$!

We can now review the projection argument. We seek to solve $X^T \approx \hat{P}X^T$, where now $\hat{P} \in \mathbb{R}^{n \times n}$. We argued that any symmetric projection matrix can be written as $\hat{P} = GG^T$ such that $G^TG = I$. The choice of notation is not at all coincidental. Using notation introduced earlier, we can now simply write:

\begin{subequations}
	\begin{align}
		X^T   &\approx GG^TX^T \\
		X   &\approx XGG^T \\
		X   &\approx XBD^\frac{1}{2}D^\frac{1}{2}B^T \\
		X   &\approx XBDB^T
	\end{align}
\end{subequations}

recovering the clustering problem introduced in (\ref{eq:kmeans_vectorized}). 

A variety of algorithms for solving PNMF with respect to different divergences have been presented in literature. Here we present only the multiplicative update algorithm for minimizing Euclidean distance. The proof of convergence regarding this algorithm, and variations can be found in \cite{yuan_projective_2005,yuan_family_2007,yang_projective_2007,yuan_projective_2009,zhang_adaptive_2012}. Given $$ \min_G ||X - GG^TX||_F^2 $$ the cost function is non-increasing, and keeps $G \ge 0$ under the update rule:

\begin{equation}
	G_{ij} \leftarrow G_{ij} \frac{2(XX^TG)_{ij}}{(GG^TXX^TG)_{ij}(XX^TGG^TG)_{ij}}
\end{equation}

One of the main benefits arriving with PNMF is reducing the number of free parameters to be learned. In the classical NMF formulation, the number of free parameters is $k(m+n)$, which makes for ambiguity. Furthermore, it is not unusual that $k(m+n) > mn$, i.e. the number of free parameters is larger than the data matrix itself. However, with the above formulation, PNMF ensures that the number of parameters to be learned is $km < mn$. 

Finally, learning only one parameter matrix $G$ means that any incoming new new data vector $\mathbf{x}$ can be \emph{predicted}, or in our context, placed into a cluster, by calculating $GG^T\mathbf{x}^T$. This is in contrast to classical NMF, where both W and H are needed. 

\subsection{Non-Negative Spectral Clustering}
\label{sec:nsc}

Nonnegative Spectral Clustering (NSC) is an implementation of NMF on the spectral clustering problem, proposed by Ding et al. in \cite{ding_nonnegative_2008}. 

NSC is a graph theoretical approach which follows naturally from the equivalence presented in Section \ref{sec:cluster_equiv}, or by the same authors in \cite{ding_equivalence_2005}. While the original paper showing the equivalence does not detail an algorithmic solution, a more involved formulation and solution strategy is presented in \cite{ding_nonnegative_2008}.

Before moving to detail NSC, let us try to interpret our problem in graph theoretical terms. We had introduced the affinity matrix $A$. We could safely think of this matrix as a weighted adjacency matrix of graph vertices all corresponding to a data observation. We know that $A$ is a symmetric matrix, so it makes sense that our graph is undirected.

Remember $k$-means is equivalent to solving the trace maximization problem given in(\ref{eq:kmeans_vec_tracemax}): 

$$
\max_B tr(X^TXBDB^T)
$$

\begin{subequations}
	\label{eq:tracemax_abdbt}
	\begin{align}
		&  \max_B tr(X^TXBDB^T) \\
		&= \max_B tr(ABDB^T)
	\end{align}	
\end{subequations}

Let us walk through the matrix product in (\ref{eq:tracemax_abdbt}). Calculating $AB$ would yield a $n \times k$ matrix, with each entry corresponding to the sum of affinity measures from a certain vertex to all vertices in a certain \emph{partition}, or groups of vertices. Concretely, if we define the affinity between $\mathbf{x}_i$ and $\mathbf{x}_j$ as $d(\mathbf{x}_i,\mathbf{x}_j)$, then $(AB)_{ij} = \sum_{\mathbf{x}_\mu \in \mathcal{C}_j} d(\mathbf{x}_i, \mathbf{x}_\mu)$. Computing $ABD$ simply normalizes this to the average affinity: $(ABD)_{ij} = \frac{1}{|\mathcal{C}_j|} \sum_{\mathbf{x}_\mu \in \mathcal{C}_j} d(\mathbf{x}_i, \mathbf{x}_\mu)$. Finally, the diagonal elements of $ABDB^T$, for each observation $\mathbf{x}_i$, correspond to the average affinity from $\mathbf{x}_i$, to all the observations in the same cluster as itself. Then $(ABDB^T)_{ii} = \frac{1}{|\mathcal{C}_p|} \sum_{\mathbf{x}_\mu \in \mathcal{C}_p} d(\mathbf{x}_i, \mathbf{x}_\mu)$ such that $\mathbf{x}_i \in \mathcal{C}_p$.

Intuitively, solving (\ref{eq:tracemax_abdbt}) one tries to assign clusters such that the average affinity to co-clustered observations is maximized. In somewhat lighter terms, we aim to maximize intra-cluster affinity. Finally, before moving on to the graph cut problem, let us recap our notation above:

\begin{subequations}
	\label{eq:tracemax_affinity}
	\begin{align}
		& \max_B tr(ABDB^T) \\
		&=\max_B  \sum_{p=1}^{k} \frac{1}{|\mathcal{C}_p|} \sum_{\mathbf{x}_i , \mathbf{x}_\mu \in \mathcal{C}_p} d(\mathbf{x}_i, \mathbf{x}_\mu) \\
		&=\max_B  \sum_{p=1}^{k} \frac{1}{|\mathcal{C}_p|} \sum_{\mathbf{x}_i , \mathbf{x}_\mu \in \mathcal{C}_p} A_{i\mu}
	\end{align}
\end{subequations}

Let us now introduce the classical graph cut problem. Given a graph $\mathcal{G}(\mathcal{V},\mathcal{E})$, a \emph{minimum cut} aims to find disjoint subsets of vertices, or partitions, $\mathcal{C}_1, \mathcal{C}_2, ..., \mathcal{C}_k$ that minimizes:

$$ J_{cut} = \sum_{1 \le p < q \le k} \sum_{\mathbf{x}_i \in \mathcal{C}_p}\sum_{\mathbf{x}_j \in \mathcal{C}_q} A_{ij} $$

This version of the problem is also known as minimum k-cuts. However, solely solving for this objective function often yields insufficient results in practice. Observe that, this objective function would be minimized by simply taking the least connected vertices as partitions. 

Normalized cuts \cite{shi_normalized_2000} is an alternative problem that aims to \emph{cut} out larger partitions, characterized by the objective function given below.

\begin{equation}
\label{eq:normalized_cut}
J_{ncut} = \sum_{1 \le p < q \le k} 
\frac{\sum_{\mathbf{x}_i \in \mathcal{C}_p} \sum_{\mathbf{x}_\mu \in \mathcal{C}_q} A_{i\mu}}{\sum_{\mathbf{x}_i \in \mathcal{C}_p} \sum_{\mathbf{x}_j \in \mathcal{V}} A_{ij}}
+ \frac{\sum_{\mathbf{x}_i \in \mathcal{C}_p} \sum_{\mathbf{x}_\mu \in \mathcal{C}_q} A_{i\mu}}{\sum_{\mathbf{x}_\mu \in \mathcal{C}_q} \sum_{\mathbf{x}_j \in \mathcal{V}} A_{\mu j}}
\end{equation}

%

Note the similarity between (\ref{eq:normalized_cut}) and (\ref{eq:tracemax_affinity}). From here, we have the intuition that this cost function must be easily vectorized. Let $\mathbf{b_p} \in \{0,1\}^n$ be an indicator for cluster $\mathcal{C}_p$, a column of the earlier cluster assignment matrix $B$. Let $\delta_i = \sum_{j \in V} A_{ij}$, and $\Delta = diag(\delta_1, \delta_2, ..., \delta_n) \in \mathbb{R}^{n \times n}$. The cost function in (\ref{eq:normalized_cut}) can be expressed as:

\begin{subequations}
	\label{eq:norm_cut_vectorized}
	\begin{align*}
		J_{ncut} &= \sum_{l=1}^{k} \frac{\mathbf{b}_l^T (\Delta - A) \mathbf{b}_l}{\mathbf{b}_l^T \Delta \mathbf{b}_l} \\
			   &= K - \sum_{l=1}^{k} \frac{\mathbf{b}_l^T A \mathbf{b}_l}{\mathbf{b}_l^T \Delta \mathbf{b}_l}
	\end{align*}
\end{subequations}

Define $\Delta^{1/2}$ as $\Delta^{1/2}_{ij} = \sqrt{\Delta_{ij}}$, and $$\Gamma = (\mathbf{b}_{1} / ||\Delta^{1/2}\mathbf{b}_{1}||, \mathbf{b}_{2} / ||\Delta^{1/2}\mathbf{b}_{2}||, ..., \mathbf{b}_{k} / ||\Delta^{1/2}\mathbf{b}_{k}|| )$$ 

The problem becomes: 

\begin{equation}
	\label{eq:nsc_tracemax}
	\max_{\Gamma^T\Delta\Gamma = I, \Gamma \ge 0} tr(\Gamma^TA\Gamma)
\end{equation}

Observe that simply taking $\Delta = I$ would have made the problem equivalent to the one introduced in Section \ref{sec:cluster_equiv}, since then $\Gamma = BD^\frac{1}{2} = G$, and the constraint $\Gamma^T\Delta\Gamma$ would become $G^TG = I$. Further observe that the existence of $\Delta$ in the equation is a prerequisite for the normalization, as in \emph{normalized} cut.

Authors of \cite{ding_nonnegative_2008} present a multiplicative update rule, along with convergence and correctness proofs for solving (\ref{eq:nsc_tracemax}) numerically:

$$ \Gamma_{ij} \leftarrow \Gamma_{ij} \sqrt{\frac{(A\Gamma)_{ij}}{(D\Gamma\alpha)_{ij}}}
\,\, (\alpha = \Gamma^TA\Gamma) $$ 

The main advantage arriving with expressing the clustering problem in a graph theoretical approach is the ability to use a variety of affinity matrices, or \emph{kernel functions}. As we presented previously, ordinary $k$-means can be thought of as a graph cut problem on the linear inner-product kernel matrix $X^TX$. However with \emph{spectral} clustering, the kernel can be extended to more powerful varieties.

\subsection{Cluster-NMF}
\label{sec:cluster_nmf}

A key shortcoming of clustering by NMF is that the data matrix, and consequently the cluster centroids are constrained to be nonnegative. By relaxing this nonnegativity constraint in \cite{ding_convex_2010}, Ding et al. strengthen the equivalence between $k$-means and NMF, and generalize the application of the technique to mixed sign data matrices.

Consider our first presentation of $k$-means in matrix factorization form, in (\ref{eq:kmeans_vectorized}): $\min_B ||X - XBDB^T||_F^2$. Here, since the data matrix nonnegative, the cluster centroid matrix $XBD$ is also constrained. The authors of \cite{ding_convex_2010} first relax the nonnegativity constraint on ordinary matrix factorization, introducing \emph{Semi-NMF}. Then, similar to a line of argument presented in Section \ref{sec:clus_kmeans}, they constrain the left factorizing matrix to convex combinations of the data matrix, introducing \emph{Convex-NMF}:

$$ \min_{1 \ge F \ge 0, B \ge 0}||X_{\pm} - X_{\pm}FB^T||^2_F $$

Here, $X_{\pm}$ is used to denote that the data matrix now has mixed signs. Note that $XF$ is now a matrix of \emph{convex combinations} of the columns of $X$. This constraint better captures the notion of centroids, if we start interpreting the $XF$ as a cluster centroid matrix and $B$ as cluster assignments. Finally, noting that the extra degree of freedom on the matrix $F$ is not necessary, \emph{Cluster-NMF} is given, in exact formulation introduced in earlier sections, but with nonnegativity of data matrix relaxed:

\begin{subequations}
	\begin{align}
		&\min_{B \ge 0} ||X_{\pm} - X_{\pm}BDB^T||_F\\
		&=\min_{G \ge 0} ||X_{\pm} - X_{\pm}GG^T||_F
	\end{align}
\end{subequations}

Note also that without the nonnegativity relaxation, this is exactly equal to the PNMF formulation introduced earlier, in that $||X^T - GG^TX^T||_F^2 = ||X - XGG^T||_F^2$. The algorithm for optimizing the objective function is also given in \cite{ding_convex_2010}, and will not be repeated here.

\subsection{Theoretical Similarities and Contrasts}

Until this point, we introduced how $k$-means can be interpreted in \emph{vectorized} form, and a matrix factorization framework. After introducing the basic notation, we demonstrated that a variety of NMF methods applied to clustering adhere to this theoretical foundation. Here, we must briefly contrast these methods to understand what underlies their differences in application and performance.

The divide between SNMF and other methods should be apparent to the reader. SNMF simply introduces a regularization embedded into the NMF problem. The method entails solving two least squares problems at each iteration. It further introduces two regularization parameters that the implementer can use to trade off  interpretability and accuracy.

The other techniques, PNMF, NSC and Cluster-NMF introduce one key feature. In the computation of these factorizations, calculating the Gram matrix $X^TX$ is necessary. As argued in Section \ref{sec:clus_kernel_kmeans}, this means the algorithms are readily kernelized and can be extended to work on $A = \Phi^T\Phi$. As key contrasts, NSC introduces the extra \emph{normalized cut} constraint, and Cluster-NMF generalizes the techniques to mixed-sign data matrices.


Another set of differences of these techniques, which this report does not cover is their extensions to more complex, asymmetric divergence (e.g. Kullback-Leibler, Itakura-Saito etc.). 

Finally, the authors of above techniques present different multiplicative update rules and algorithms that serve the unique constraints. These variations in implementation have been shown to yield diverse results in real-world applications.

\section{Conclusion}
\label{sec:conclusion}

In this work, we aimed to present NMF, and its applications to the clustering problem.

Potential applications of NMF in clustering were discussed with the early research of the technique. While initial findings with NMF yielded parts-based representations and sparser encodings, dense and holistic representations were also demonstrated. This led to the pursuit of factorizing for sparser representations.

As presented above, sparser representations have a direct link to interpretability in clustering. When one of the factorizing matrices is to be used for cluster assignment, interpreting that matrix requires row-wise sparsity, or near-orthogonality. 

Before moving to discuss \emph{sparser} techniques, we established that a direct link between  \emph{soft} $k$-means and NMF could be drawn. In later sections, we extended this argument to stronger relationships shown between spectral clustering and matrix factorization equivalents.

Finally, we covered an array of matrix factorization methods that can be used in clustering. While each such presentation provided unique \emph{features}, we drew on their similarities and demonstrated that nearly all share a theoretical basis. 

It must be noted that the NMF problem and extensions to various application domains and mathematical consequents are well studied. Extensions to higher order factorizations such as Quadratic NMF, tensor decompositions, Bayesian inference exist, with consequences relating to the clustering problem. These studies were left out of scope for this report, but may well constitute a next step in research for the reader.

\section*{Acknowledgement}

The author would like to thank Taylan Cemgil for the guidance and fruitful discussions during the writing of this report.

\appendix
\section{Some Useful Results From Linear Algebra}
\label{app:a}

This section contains some useful proofs in linear algebra we often refer to in this document.

\begin{lemma}
	\label{lemma:proj}
	Let $W$ be a matrix of orthonormal columns, such that $W^TW = I$. Then $WW^T$ is a projection matrix.
\end{lemma}

\begin{proof}
Recall that a projection matrix is defined as a square matrix $P$ such that $PP = P$ holds. Then:
\begin{subequations}
	\begin{align}
	WW^TWW^T &= W(W^TW)W^T  \\
	&= WIW^T  \text{\,\,(as $W^TW = I$ by definition)} \\
	&= WW^T 
	\end{align}
\end{subequations}

\end{proof}


\begin{lemma} \label{lemma:eigendecomp}
	For any symmetric projection matrix $P \in \mathbb{R}^{m \times m}$ of rank $k$ there exists $G \in \mathbb{R}^{m \times k}$ such that $P = GG^T$ and $G^TG = I$.
\end{lemma}

\begin{proof}
	It is known that any real symmetric matrix has an eigendecomposition, written in the form
	
	$$ P = Q \Lambda Q^T $$
	
	where $Q \in \mathbb{R}^{m \times m}$ is a matrix with the eigenvectors of $P$ arranged as its columns, and $\Lambda$ is a diagonal matrix with corresponding eigenvalues arranged on its diagonal, in descending order. 
	
	Note that the eigenvalues of a projection matrix can only be 1 or 0. To derive this result, take the $\mathbf{v}$ as an eigenvector and $\lambda$ as an eigenvalue of $P$. Then $P\mathbf{v} = \lambda \mathbf{v} = PP\mathbf{v} = P \lambda \mathbf{v} = \lambda ^2 \mathbf{v}$, hence $\lambda^2 = \lambda$, by which we have $\lambda \in \{0,1\}$.
	
	Finally, observe that the number of nonzero eigenvalues is equal to the rank of a symmetric matrix. Then the matrix $\Lambda$ takes the form:
	
	$$ \Lambda = \left[ \begin{array}{cc}
								I_k & 0  \\
								0 & 0  \end{array} \right] 
	$$
	
	where $I_k$ is the $k \times k$ identity matrix. Then simply defining $G$ as the first $k$ columns of $Q$, i.e. $G = (\mathbf{q}_1,\mathbf{q}_2,...,\mathbf{q}_k)$ one can see $P = GIG^T = GG^T$. Since the columns of Q are orthogonal and unit norm vectors, we have $G^TG=I$. 
\end{proof}

\begin{lemma}
	\label{lemma:tracemax}
	$\min_B  ||X - XBDB^T||_F^2 = \min_B tr(X^TX) - tr(X^TX B D B^T)$
\end{lemma}

\begin{proof}
We start by writing the Frobenius norm as the trace:

$$
\min_B tr((X - XBDB^T)^T(X - XBDB^T))
$$

First, we introduce $D^{\frac{1}{2}}$, where $D^{\frac{1}{2}}_{ij} = \sqrt{D_{ij}}$. In the next few equations, we will make frequent use of some key properties. First of all, as $D$ (and $D^{\frac{1}{2}}$) is diagonal, $D^T = D$. Secondly, the trace operator is invariant under cyclical permutations, i.e. $tr(ABC) = tr(CAB) = tr(BCA)$. Finally, note that by definition of our clustering notation, $BD^{\frac{1}{2}}$ has orthonormal columns, or $(BD^{\frac{1}{2}})^T(BD^{\frac{1}{2}}) = I$. We can now go through the following steps in showing the equivalence:

\begin{subequations}
	\begin{align*}
		& \min_B tr((X - XBDB^T)^T(X - XBDB^T)) \\
		&= \min_B tr((X^T - BDB^TX^T)(X - XBDB^T)) \\
		&= \min_B tr(X^TX - X^TXBDB^T - BDB^TX^TX +BDB^TX^TXBDB^T) \\
		&= \min_B tr(X^TX) - tr(X^TXBDB^T) - tr(BDB^TX^TX) + tr(BDB^TX^TXBDB^T) \\
		&= \min_B tr(X^TX) - 2tr(X^TXBDB^T) + tr(BD^{\frac{1}{2}}D^{\frac{1}{2}}B^TX^TXBD^{\frac{1}{2}}D^{\frac{1}{2}}B^T) \\
		&= \min_B tr(X^TX) - 2tr(X^TXBDB^T) + tr((BD^{\frac{1}{2}})(BD^{\frac{1}{2}})^T X^TX BD^{\frac{1}{2}}(BD^{\frac{1}{2}})^T) \\
		&= \min_B tr(X^TX) - 2tr(X^TXBDB^T) + tr((BD^{\frac{1}{2}})^T X^TX BD^{\frac{1}{2}}) \\
		&= \min_B tr(X^TX) - 2tr(X^TXBDB^T) + tr(X^TX BD^{\frac{1}{2}} (BD^{\frac{1}{2}})^T ) \\
		&= \min_B tr(X^TX) - 2tr(X^TXBDB^T) + tr(X^TX BDB^T ) \\
		&= \min_B tr(X^TX) - tr(X^TXBDB^T)
	\end{align*}
\end{subequations}

\end{proof}

\bibliography{zotero,mendeley}{}

\begin{thebibliography}{10}

\bibitem{dhillon_kernel_2004}
Inderjit~S. Dhillon, Yuqiang Guan, and Brian Kulis.
\newblock Kernel {K}-means: {Spectral} {Clustering} and {Normalized} {Cuts}.
\newblock In {\em Proceedings of the {Tenth} {ACM} {SIGKDD} {International}
  {Conference} on {Knowledge} {Discovery} and {Data} {Mining}}, {KDD} '04,
  pages 551--556, New York, NY, USA, 2004. ACM.

\bibitem{dhillon_unified_2004}
Inderjit~S. Dhillon, Yuqiang Guan, and Brian Kulis.
\newblock {\em A unified view of kernel k-means, spectral clustering and graph
  cuts}.
\newblock Citeseer, 2004.

\bibitem{ding_nonnegative_2008}
C.~Ding, Tao Li, and M.I. Jordan.
\newblock Nonnegative {Matrix} {Factorization} for {Combinatorial}
  {Optimization}: {Spectral} {Clustering}, {Graph} {Matching}, and {Clique}
  {Finding}.
\newblock In {\em Eighth {IEEE} {International} {Conference} on {Data}
  {Mining}, 2008. {ICDM} '08}, pages 183--192, December 2008.

\bibitem{ding_convex_2010}
Chris Ding, Tao Li, and Michael~I. Jordan.
\newblock Convex and semi-nonnegative matrix factorizations.
\newblock {\em Pattern Analysis and Machine Intelligence, IEEE Transactions
  on}, 32(1):45--55, 2010.

\bibitem{ding_equivalence_2005}
Chris~HQ Ding, Xiaofeng He, and Horst~D. Simon.
\newblock On the {Equivalence} of {Nonnegative} {Matrix} {Factorization} and
  {Spectral} {Clustering}.
\newblock In {\em {SDM}}, volume~5, pages 606--610. SIAM, 2005.

\bibitem{hoyer_non-negative_2004}
Patrik~O. Hoyer.
\newblock Non-negative matrix factorization with sparseness constraints.
\newblock {\em The Journal of Machine Learning Research}, 5:1457--1469, 2004.

\bibitem{kim_sparse_2006}
Hyunsoo Kim and Haesun Park.
\newblock Sparse non-negative matrix factorizations via alternating
  non-negativity-constrained least squares.
\newblock 2006.

\bibitem{kim_sparse_2007}
Hyunsoo Kim and Haesun Park.
\newblock Sparse non-negative matrix factorizations via alternating
  non-negativity-constrained least squares for microarray data analysis.
\newblock {\em Bioinformatics}, 23(12):1495--1502, June 2007.

\bibitem{kim_nonnegative_2008}
Hyunsoo Kim and Haesun Park.
\newblock Nonnegative matrix factorization based on alternating nonnegativity
  constrained least squares and active set method.
\newblock {\em SIAM Journal on Matrix Analysis and Applications},
  30(2):713--730, 2008.

\bibitem{kim_sparse_2008}
Jingu Kim and Haesun Park.
\newblock Sparse {Nonnegative} {Matrix} {Factorization} for {Clustering}.
\newblock 2008.

\bibitem{lee_learning_1999}
Daniel~D. Lee and H.~Sebastian Seung.
\newblock Learning the parts of objects by non-negative matrix factorization.
\newblock {\em Nature}, 401(6755):788--791, October 1999.

\bibitem{lee_algorithms_2001}
Daniel~D. Lee and H.~Sebastian Seung.
\newblock Algorithms for {Non}-negative {Matrix} {Factorization}.
\newblock In {\em In {NIPS}}, pages 556--562. MIT Press, 2001.

\bibitem{lin_projected_2007}
Chih-Jen Lin.
\newblock Projected gradient methods for nonnegative matrix factorization.
\newblock {\em Neural computation}, 19(10):2756--2779, 2007.

\bibitem{lloyd_least_1982}
Stuart Lloyd.
\newblock Least squares quantization in {PCM}.
\newblock {\em Information Theory, IEEE Transactions on}, 28(2):129--137, 1982.

\bibitem{macqueen_methods_1967}
James MacQueen and {others}.
\newblock Some methods for classification and analysis of multivariate
  observations.
\newblock In {\em Proceedings of the fifth {Berkeley} symposium on mathematical
  statistics and probability}, volume~1, pages 281--297. Oakland, CA, USA.,
  1967.

\bibitem{paatero_positive_1994}
Pentti Paatero and Unto Tapper.
\newblock Positive matrix factorization: {A} non-negative factor model with
  optimal utilization of error estimates of data values.
\newblock {\em Environmetrics}, 5(2):111--126, June 1994.

\bibitem{shi_normalized_2000}
Jianbo Shi and Jitendra Malik.
\newblock Normalized cuts and image segmentation.
\newblock {\em Pattern Analysis and Machine Intelligence, IEEE Transactions
  on}, 22(8):888--905, 2000.

\bibitem{Shlens2014}
Jonathon Shlens.
\newblock {A Tutorial on Principal Component Analysis}.
\newblock {\em arXiv:1404.1100 [cs, stat]}, April 2014.

\bibitem{smola_tutorial_2004}
Alex~J. Smola and Bernhard Schölkopf.
\newblock A tutorial on support vector regression.
\newblock {\em Statistics and Computing}, 14(3):199--222, August 2004.

\bibitem{yang_projective_2007}
Zhirong Yang, Zhijian Yuan, and Jorma Laaksonen.
\newblock Projective non-negative matrix factorization with applications to
  facial image processing.
\newblock {\em International Journal of Pattern Recognition and Artificial
  Intelligence}, 21(08):1353--1362, 2007.

\bibitem{yuan_family_2007}
Zhijian Yuan and E.~Oja.
\newblock A family of modified projective {Nonnegative} {Matrix}
  {Factorization} algorithms.
\newblock In {\em 9th {International} {Symposium} on {Signal} {Processing} and
  {Its} {Applications}, 2007. {ISSPA} 2007}, pages 1--4, February 2007.

\bibitem{yuan_projective_2005}
Zhijian Yuan and Erkki Oja.
\newblock Projective {Nonnegative} {Matrix} {Factorization} for {Image}
  {Compression} and {Feature} {Extraction}.
\newblock In Heikki Kalviainen, Jussi Parkkinen, and Arto Kaarna, editors, {\em
  Image {Analysis}}, number 3540 in Lecture {Notes} in {Computer} {Science},
  pages 333--342. Springer Berlin Heidelberg, January 2005.

\bibitem{yuan_projective_2009}
Zhijian Yuan, Zhirong Yang, and Erkki Oja.
\newblock Projective nonnegative matrix factorization: {Sparseness},
  orthogonality, and clustering.
\newblock {\em Neural Process. Lett}, 2009.

\bibitem{zha_spectral_2001}
Hongyuan Zha, Xiaofeng He, Chris Ding, Ming Gu, and Horst~D. Simon.
\newblock Spectral relaxation for k-means clustering.
\newblock In {\em Advances in neural information processing systems}, pages
  1057--1064, 2001.

\bibitem{zhang_adaptive_2012}
He~Zhang, Zhirong Yang, and Erkki Oja.
\newblock Adaptive {Multiplicative} {Updates} for {Projective} {Nonnegative}
  {Matrix} {Factorization}.
\newblock In Tingwen Huang, Zhigang Zeng, Chuandong Li, and Chi~Sing Leung,
  editors, {\em Neural {Information} {Processing}}, number 7665 in Lecture
  {Notes} in {Computer} {Science}, pages 277--284. Springer Berlin Heidelberg,
  January 2012.

\end{thebibliography}
\bibliographystyle{plain}

\end{document}